\documentclass[twoside,11pt]{article}

\usepackage{jmlr2e-mod}

\usepackage{algorithm2e}
\usepackage{amsmath}

\firstpageno{1}

\begin{document}

\title{Simple and Fast Algorithms for Interactive Machine Learning with Random Counter-examples}

\author{Jagdeep Bhatia \\
200430@whrhs-stu.org \\
Watchung Hills Regional High School\\
Warren, NJ 07059, USA \\}

\maketitle

\begin{abstract}
This work describes simple and efficient algorithms for interactively learning non-binary concepts in the learning from random counter-examples (LRC) model. Here, learning takes place from random counter-examples that the learner receives in response to their \textit{proper equivalence queries}. In this context, the learning time is defined as the number of counter-examples needed by the learner to identify the target concept. Such learning is particularly suited for online ranking, classification, clustering, etc., where machine learning models must be used before they are fully trained. 

We provide two simple LRC algorithms, deterministic and randomized, for exactly learning non-binary target concepts for any concept class $H$. We show that both of these algorithms have an $\mathcal{O}(\log{}|H|)$ asymptotically optimal average learning time. This solves an open problem on the existence of an efficient LRC randomized algorithm while simplifying and generalizing previous results. We also show that the expected learning time of any arbitrary LRC algorithm can be upper bounded by $\mathcal{O}(\frac{1}{\epsilon}\log{\frac{|H|}{\delta}})$, where $\epsilon$ and $\delta$ are the allowed learning error and failure probability respectively. This shows that LRC interactive learning is at least as efficient as non-interactive Probably Approximately Correct (PAC) learning. Our simulations show that in practice, these algorithms outperform their theoretical bounds.

\end{abstract}


\section{Introduction}

Machine learning has made great advances in fields such as image recognition and natural language processing. However, it currently requires large sets of training data upfront. This is a problem because often the amount of data available is small or sometimes the machine learning models must be used before they are fully trained. In these cases, interactive learning, which involves training machine learning models through interaction between a learner and a teacher can be very beneficial. For example, in personalized learning with one-on-one interaction, teachers are better able to assess their students' strength and weakness and provide individualized instruction. Companies such as Netflix and Amazon also use these kinds of interactive algorithms to learn the preferences of their customers and provide them with engaging content.


There are numerous types of interactive learning frameworks, such as that of active learning~\citep{burr2012}. Our work, however, uses the recently proposed framework of exact learning through random counter-examples (LRC) 
presented by Angluin and Dohrn~\citep{angluin2017power}. In this environment, the learning process may be viewed as a game between a \textit{teacher} and a \textit{learner}, where the goal of the learner is to identify a \textit{target concept} chosen by the teacher from a set of hypotheses $H$ called a \textit{concept class}. The learner learns through \textit{proper equivalence queries}~\citep{angluin2017power}, which means that in every round of the learning process, the learner queries the teacher by selecting a hypothesis from the concept class. The teacher either indicates that the learner has identified the target concept, or randomly reveals one of the mistakes the learner's hypothesis makes, called a \textit{counter-example}. This process continues until the learner correctly identifies the target concept, and the learner's goal is to minimize the number of rounds needed to do this.

While other interactive learning models assume adversarial teachers and do worst-case analysis~\citep{angluin1988queries, bf-pgrf-72, zadeh2017, littlestone1988learning}, the LRC framework proposes a more helpful  teacher for which average case analysis can be done. This is more realistic because in most practical applications, the teacher will not be trying to hinder the learning process. 
Additionally, learning in the LRC model is proper, meaning that the learner picks their hypotheses from their concept class. This is also more realistic because in applications such as personalized learning, learners with different concept classes can represent learners of different skill and prior knowledge levels. With improper learning, however, learners cannot be distinguished in this way. 

The main contribution of our work is the simplification, generalization, and solution of an open problem from previous research in 
the LRC model~\citep{angluin2017power}. We also provide a mathematical proof that interactive LRC learning 
is at least as efficient as non-interactive Probably Approximately Correct (PAC) learning~\citep{kearns1994introduction} regardless of the learning algorithm used by the interactive learner. In particular our results are as follows:
\begin{itemize}
\item {\bf Majority learning algorithm:} We provide a simple interactive learning algorithm based on majority vote which learns in the best possible $\mathcal{O}(\log{}|H|)$ expected number of rounds. This algorithm works by simply picking, in each round, a hypothesis in the concept class that has the highest "majority" score among the hypotheses that are consistent with, or do not contradict, the counter-examples seen so far. Thus, this algorithm can be seen as a generalization of the well known halving algorithm based on majority voting~\citep{angluin1988queries, bf-pgrf-72, littlestone1988learning} to the LRC setting. 
The majority  algorithm requires lower computation time  and is a significant simplification over the previously known Max-Min  algorithm~\citep{angluin2017power} for learning binary hypotheses, which was also shown to be asymptotically optimal in the LRC setting, but is much more complex.

\item {\bf Randomized learning algorithm:} We solve the open problem posed by Angluin and Dohrn~\citep{angluin2017power} that asks if there exists 
a randomized learning algorithm with the same asymptotically optimal  $\mathcal{O}(\log{}|H|)$ bound on expected number of rounds when the teacher draws the target concept from a known probability distribution over all hypothesis in $H$ at the beginning of the learning process. We show that such an algorithm does exist, and that the bound is achieved when the learner also draws consistent hypotheses from the same probability distribution conditioned on the previous sequence of learner's hypotheses and the teacher's counter-examples. This algorithm also requires lower computation time than the Max-Min algorithm.

\item {\bf Upper bound on the learning time of an arbitrary learning algorithm:} We prove that with probability greater than $1- \delta$ and with error less than $\epsilon$, the expected number of rounds for an  arbitrary LRC algorithm is upper bounded by $\mathcal{O}(\frac{1}{\epsilon}\log{\frac{|H|}{\delta}})$. This result holds for any arbitrary small $\epsilon$ and $\delta$ values. It also establishes that LRC learning, regardless of the learning algorithm, is at least as efficient as non-interactive Probably Approximately Correct (PAC) learning~\citep{kearns1994introduction}.

\item {\bf Generalization to non-binary hypotheses:} We extend the original work on LRC~\citep{angluin2017power} which is for binary concept classes $H$, more generally to concept classes with arbitrary values.

\item {\bf Performance evaluation:} We show with simulations that in practice, the Majority and Arbitrary learning algorithms outperform their worst case bounds.
\end{itemize}

\section{Related Work}
\label{related}
Interactive learning is typically characterized by learning that takes place through query and response.
One of the most common models of interactive learning is active learning~\citep{burr2012}, where the learner queries the teacher for the labels of sample points. Active learning is commonly used in settings where labeling costs are high and therefore it can be much more cost-effective to selectively label the samples as opposed to labeling all the samples upfront.

Our work, however, pertains to the interactive learning setting of learning with random counter-examples~\citep{angluin2017power} using equivalence queries. In this setting, the learner is not allowed to query the teacher directly for sample points. Instead, the teacher randomly selects which sample points to label (i.e. gives random counter-examples) based on where the learner's hypothesis is wrong. Interactive learning that deals with more general membership, equivalence, and related queries, was initially proposed in the seminal work of Angluin~\citep{angluin1988queries}.

In a related setting, Littlestone ~\citep{littlestone1988learning} developed efficient learning algorithms for minimizing the number of mistakes when learning certain boolean functions including k-DNF. 
Additionally, there has been significant work on interactive learning with equivalence queries that deals with specific geometric classes such as hyperplanes or axis-aligned boxes~\citep{massturan94}. The work of Maass and Tur{\'a}n~\citep{massturan90, massturan92} is on the complexity of interactive learning including lower bounds on the number of membership and equivalence queries required for exact learning. Many complexity results can be found in the work of Angluin~\citep{angluin2004}.

Recent work in interactive learning with equivalence queries pertains to clustering, ranking, and classification~\citep{zadeh2017}. In this line of work~\citep{joachims2002, zadeh2017}, one of the goals is to quickly learn the preferred ranking of a list of items in an online setting. This is motivated by applications in personalized web search and information retrieval systems where the learning algorithms learn their users' preferences through online feedback in the form of click behavior. In another line of work which is on  interactive learning for clustering, the goal is to learn a user's preferred clustering of a set of objects in an online fashion~\citep{awasthi2017, balcan2008, zadeh2017}.
An application of this work includes identifying communities in social networks.

Learning in the LRC framework~\citep{angluin2017power}, on which our work is based, is proper since the learner's equivalence queries are required to be from their concept class $H$.
This makes LRC different from other recent work~\citep{zadeh2017} as in that setting the learner's hypotheses are not necessarily drawn from the concept class $H$. This is also what makes our work different from the previous work on halving algorithms ~\citep{littlestone1988learning, bf-pgrf-72, angluin1988queries} based on majority vote. The LRC model~\citep{angluin2017power} that we use is also unique in that the teacher's counter-examples are not given arbitrarily, but are randomly drawn from a fixed probability distribution conditioned on the set of all possible counter-examples.

\section{Learning Models}
\label{learningmodels}
\linespread{1.5}
The primary learning model used in this work is that of learning through random counter-examples (LRC)~\citep{angluin2017power}. As LRC is only defined for exact learning, we propose a new learning model that we call \textit{PAC-LRC} which models approximate learning in the LRC setting.

In LRC, the learner's concept class can be viewed as a $n \times m$ matrix $H$, whose rows (denoted by $h$) represent the set of learners \textit{hypotheses} and whose columns (denoted by $X$) represent the set of examples (samples). The \textit{target concept} is one of the rows of the matrix, $h^* \in H$. Additionally, the entries of matrix $H$ represent the values assigned by each hypothesis $h \in H$ to each example $x \in X$ (denoted by function $h(x)$). In this framework, the learner learns through \textit{proper equivalence queries}~\citep{angluin2017power}. This means that in every round of the learning process, the learner queries the teacher by selecting a hypothesis $h \in H$. The teacher either indicates that $h$ is the target concept (and hence the learning is complete), or reveals a \textit{counter-example} $x$ and the value of $h^*(x)$ for $x \in X$ on which the target concept differs with $h$, i.e. $h(x) \neq h^*(x)$. Moreover, there is a known probability distribution $\mathbb{P}$ over $X$ and the teacher's counter-example $x$ is drawn from the probability distribution $\mathbb{P}(h,h^*)$ (denoted by $x \sim \mathbb{P}(h,h^*)$) which is defined as $\mathbb{P}$ conditioned on the event $h(x) \neq h^*(x)$. Upon receiving the counter-example, the learner selects another hypothesis $h \in H$ for the next round and this continues until $h = h^*$. The learner's goal is to minimize the number of rounds needed for learning the target concept $h^*$. In this work, the LRC learning model is used for both the Majority algorithm and the Randomized learning algorithms. However, for the Randomized learning algorithm it is additionally assumed that the \textit{target concept} is randomly drawn by the teacher at the beginning of the learning process using a known probability distribution. 

We also define a model \textit{PAC-LRC}, an extension of  LRC, for approximate learning with random counter-examples. For this we introduce two additional parameters:
$\epsilon$, the allowed error, and $\delta$, the allowed failure probability.  In PAC-LRC the goal is to approximately learn the target concept  with probability at least $1-\delta$ and with error no more than $\epsilon$. For this we define a $\epsilon$-\textit{bad} hypothesis to be a hypothesis that differs with the target concept in a region that has total probability at least $\epsilon$. As in LRC, learning in PAC-LRC proceeds in rounds. In each round in which the learner presents a  $\epsilon$-\textit{bad} hypothesis a  randomly drawn counter-example is returned to the learner and the learning continues. Otherwise, the learner's hypothesis is accepted and the learning ends. Learning in PAC-LRC   may also end if the probability of elimination of all $\epsilon$-\textit{bad} hypotheses in concept class $H$ exceeds $1-\delta$.
The PAC-LRC model is inspired by the well known probably approximately correct or PAC model~\citep{kearns1994introduction}, where generally $\delta, \epsilon \ll \frac12$. However, unlike the LRC and PAC-LRC model, in the PAC model the probability distribution on examples is unknown to the learner. 

\section{Algorithmic Results}
\label{results}
In section~\ref{ssprelim}, some definitions used throughout the paper are explained. In section~\ref{ssmla}, the Majority learning algorithm is presented and shown to be asymptotically optimal in the LRC model. In section~\ref{ssrla}, the Randomized learning algorithm is presented and is also shown to be asymptotically optimal in the LRC model. In section~\ref{ssala}, an upper bound is derived for the learning time of an Arbitrary learning algorithm in the PAC-LRC setting.


\subsection{Preliminaries}
\label{ssprelim}
\begin{definition}
\label{hdef}
$H$ the learner's concept class is a $n\times m$ matrix, for positive integers $n$ and $m$, with no duplicated rows or columns. $H$ can be thought of as a set of hypotheses. $|H|$ denotes the number of hypotheses in $H$.
\end{definition}

\begin{definition}
\label{xdef}
$X$ denotes the set of columns of matrix $H$.
\end{definition}

\begin{definition}
\label{lildef}
$h \in H$ denotes a hypothesis or row in matrix $H$.
For $x \in X$ the function $h(x) \in \mathbb Z_{\ge 0}$ denotes the value of row $h$ at column $x$ in $H$. 
Here $\mathbb Z_{\ge 0}$ refers to the set of non-negative integers.
\end{definition}

\begin{definition}
\label{pdef1}
$\mathbb{P}$ denotes a probability distribution over $X$. For $x \in X$, $\mathbb{P}(x)>0$ denotes the probability that $x$ is drawn from $\mathbb{P}$, or $x \sim \mathbb{P}$, and
$\mathbb{P}(S) = \sum_{x \in S}{\mathbb{P}(x)}$. 
\end{definition}

\begin{definition}
\label{ddef}
Define $D(h_1, h_2)$ to be the set of all columns on which $h_1 \in H$ and $h_2 \in H$ have different values. More formally, $D(h_1, h_2) = \{x\in X \mid h_1(x) \neq h_2(x) \}$.
\end{definition}

\begin{definition}
\label{pdef}
$\mathbb{P}(h_1,h_2)$ is defined as a probability distribution over $D(h_1, h_2)$, and is the result of conditioning $\mathbb{P}$ on the event $h_1(x) \neq h_2(x)$.
$\mathbb{P}[x \mid h_1(x) \neq h_2(x) ]$ is defined as the individual
probability of drawing element $x \in D(h_1, h_2)$ from distribution $\mathbb{P}(h_1,h_2)$.
\end{definition}


\subsection{Majority Learning Algorithm}
\label{ssmla}
In this section, the Majority learning algorithm is proposed and mathematically analyzed and its performance is shown to be asymptotically optimal in the LRC learning model.

\begin{definition}
\label{sec4majdef}	
\(\operatorname{MAJ}_H\) is a hypothesis constructed by setting the value in each of its columns to the most frequent element in the corresponding column of matrix $H$. Ties are broken in favor of the smaller element.
Note that it is possible that $\operatorname{MAJ}_H \notin H$.
\end{definition}

\begin{definition}
\label{sec4hhatdef}
The best majority hypothesis \(\hat h\) of $H$ is any hypothesis in \(H\) that maximizes
\(\mathbb{P}(h(x) = \operatorname{MAJ}_H(x))\).  Ties are broken in favor of the smallest $h$ in a lexicographic sort by the values of $h(x)$. We  consider $h_1 \in H$ lexicographically smaller than $h_2 \in H$, if there exists a column $i$ such that $h_1(x_i) < h_2(x_i)$ and $h_1(x_j) = h_2(x_j)$ for all columns $j < i$. 
\end{definition}


\begin{algorithm}[h!]
\label{alg1}
 \While{true}{
 	Pick \(\hat h\) to be a best majority hypothesis in $H$. \\ 
	Let $x$ be the counter-example returned by the teacher for \(\hat h\).\\
	\If{ there is no such counter-example}
	{Output $\hat h$. \\}
	\Else
	{Eliminate the set of hypotheses $\{h \in H \mid h(x) \neq h^*(x)\}$ from $H$.\\}
 }
 \caption{Majority Learning Algorithm}

\end{algorithm}

The analysis for the Majority learning algorithm (Algorithm~\ref{alg1}) is as follows. Lemma~\ref{sec4lem1} and Lemma~\ref{sec4lem2} establish the performance of the algorithm for any particular round. In these Lemmas, $H$ denotes the set of consistent hypotheses (that do not contradict the counter-examples seen so far) and $\hat h$ denotes the learner's choice of hypothesis for the round being considered. Lemma~\ref{sec4lem1} shows that the probability that the teacher's counter-example $x$ is a majority element in $\hat h$, or $\hat h(x) = \operatorname{MAJ}_H(x)$, is at least $\frac{1}{2}$. Lemma~\ref{sec4lem2} shows that a counter-example drawn by the teacher will eliminate at least $\frac{1}{4}$ fraction of the remaining hypotheses in expectation.  
Through an example we show that $\frac{1}{4}$ is the best possible per round fraction that can be guaranteed for the Majority algorithm in general.
Theorem~\ref{sec4t1} shows that the Majority learning algorithm has an $\mathcal{O}(\log{}|H|)$ expected learning time which is asymptotically the best possible in the LRC model~\citep{angluin2017power}. 
Finally, Theorem~\ref{fast} bounds its per round computation time.

\begin{lemma}
\label{sec4lem1}
Let \(\hat h\) be the best majority hypothesis selected by the Majority algorithm. Let \(h^* \neq \hat h\) be any hypothesis in $H$. Let $A$ be the event  $\{ x : \hat h(x) \neq h^*(x) \wedge x \in X \}.$ Then, $\mathbb{P}(h^*(x) \neq \operatorname{MAJ}_H(x) \mid A) \ge \frac12$.
\end{lemma}
\begin{proof}
Assume for sake of contradiction that the lemma is not true. Then:
 \[\mathbb{P}(h^*(x) \neq \operatorname{MAJ}_H(x) \mid A) < \frac12 \mbox{ which implies }   \mathbb{P}(h^*(x) = \operatorname{MAJ}_H(x) \mid A) > \frac12.\]
 By definition of $A$ we have \(\hat h(x) \neq h^*(x)\) for all  $x \in A$. Thus given event $A$, for any $x$ for which $h^*(x) = \operatorname{MAJ}_H(x)$, we must have $\hat h(x) \neq \operatorname{MAJ}_H(x)$. Thus
$$ \mathbb{P}(\hat h(x) \neq \operatorname{MAJ}_H(x) \mid A) \ge \mathbb{P}(h^*(x) = \operatorname{MAJ}_H(x) \mid A) >  \frac12.$$
However,
\[\mathbb{P}(\hat h(x) \neq \operatorname{MAJ}_H(x) \mid A) > \frac12 \mbox{ implies }  \mathbb{P}(\hat h(x) = \operatorname{MAJ}_H(x) \mid A) < \frac12,\]
From $\mathbb{P}(h^*(x) = \operatorname{MAJ}_H(x) \mid A) > \frac12 \mbox{ and } \mathbb{P}(\hat h(x) = \operatorname{MAJ}_H(x) \mid A) < \frac12$, it follows that 
\[ 
 \mathbb{P}(h^*(x) = \operatorname{MAJ}_H(x) \mid A)
 > \mathbb{P}(\hat h(x) = \operatorname{MAJ}_H(x) \mid A),
\]
Since $\hat h(x) = h^*(x) $, for any $x \in X$  not in A we have
\[
 \mathbb{P}(h^*(x) = \operatorname{MAJ}_H(x))
 > \mathbb{P}(\hat h(x) = \operatorname{MAJ}_H(x)) .
\]
Since $h^* \in H$, it follows that 
$$\max_{h \in H} \mathbb{P}(h(x) = \operatorname{MAJ}_H(x)) >\mathbb{P}(\hat h(x) = \operatorname{MAJ}_H(x)).$$   This contradicts the definition of $\hat h$ (Definition \ref{sec4hhatdef}).
\end{proof}

\begin{lemma}
\label{sec4lem2}
Fix any hypothesis \(h^* \in H\), and let  \(\hat h\) be a best majority hypothesis selected by the Majority algorithm.  Let $A$ be the event  $\{ x : \hat h(x) \neq h^*(x) \wedge x \in X \}.$  For counter-example $x \sim \mathbb{P}(\hat h, h^*)$, the expected number of hypotheses
\(h \in H\) with \(h(x) \neq h^*(x)\) is at least \(|H|/4\). Thus, in expectation at least \(|H|/4\) hypotheses are eliminated by the counter-example.  
\end{lemma}
\begin{proof}
Consider counter-example $x \sim \mathbb{P}(\hat h, h^*)$ chosen in response to $\hat h$.  
By definition of $\operatorname{MAJ}_H $ (Definition~\ref{sec4majdef}) it follows that  $$ |\{ h \mid h(x)  = \operatorname{MAJ}_H(x) \}| \ge  |\{ h \mid h(x) = h^*(x) \}|.$$  
Consider a counter-example for which $h^*(x) \neq \operatorname{MAJ}_H(x)$. For such an $x$ we have that 
$$ \{ h \mid h(x)  = \operatorname{MAJ}_H(x) \} \subseteq \{ h \mid h(x) \neq h^*(x) \}$$
and therefore 
$$ |\{ h \mid h(x) \neq h^*(x) \}| \ge  |\{ h \mid h(x)  = \operatorname{MAJ}_H(x) \}| \ge |\{ h \mid h(x) = h^*(x) \}|.$$
Since, $H = \{ h \mid h(x) \neq h^*(x) \} \bigcup \{ h \mid h(x) = h^*(x) \}$ it follows that 
$$ |\{ h \mid h(x) \neq h^*(x) \}| \ge \frac{|H|}{2}.$$
Note that $ \{ h \mid h(x) \neq h^*(x) \}$ are exactly the  hypotheses in $H$ that get eliminated by counter-example $x$. This shows that a counter-example $x$ for which $h^*(x) \neq \operatorname{MAJ}_H(x)$ eliminates $|H|/2$ of the hypotheses in $H$.

By Lemma~\ref{sec4lem1}, $\mathbb{P}(h^*(x) \neq \operatorname{MAJ}_H(x) \mid A) \ge \frac12$. Thus,  with probability at least $1/2$, the counter-example chosen in response to $\hat h$ satisfies   $\mathbb{P}(h^*(x) \neq \operatorname{MAJ}_H(x))$.   
Thus, it follows that in expectation at least \(1/2 \cdot |H|/2 = |H|/4\) hypotheses are eliminated by a counter-example chosen in response to $\hat h$. 

\end{proof}

\begin{theorem}
\label{sec4t1}
The Majority learning algorithm (Algorithm~\ref{alg1}) only needs to see an expected $\log_{\frac{4}{3}} |H|$ or $\mathcal{O}(\log{}|H|)$ counter-examples to learn $h^*$ in the general case.
\end{theorem}

\begin{proof}
The proof which is based on induction  follows along the line of the proof of Theorem 21  by Angluin and Dohrn~\citep{angluin2017power}. With $T(n)$ defined for the Majority algorithm as in~\citep{angluin2017power}, as the worst case expected number of queries for any concept class  $H$ with $|H| = n$ hypotheses, we re-define predicate $P(n): T(n) \le \log_{\frac{4}{3}} n$.  Note that the base case $P(1): T(1) \le 0$ trivially holds since no queries are needed when the learner's concept class has only one hypothesis.  Also, $P(2): T(2) \le \log_{\frac{4}{3}} 2$ since at most one query is needed when the learner's concept class has two hypotheses and therefore $T(2) \le 1 < \log_{\frac{4}{3}} 2$. Proceeding as in~\citep{angluin2017power} let us assume $P(r)$ is true for all positive
integers $r \le n$  for some $n \ge 2$. Let $H$ be any concept class with $|H| = n+1$ hypotheses. Let $R$ as defined in~\citep{angluin2017power} be the number of remaining consistent hypotheses that do not get eliminated by the teachers counter-example in a round of the Majority algorithm applied to $H$. Then as in ~\citep{angluin2017power}, 
$$T(n+1) \le 1 + \sum_{r=1}^ n \left(\mathbb{P}[R=r] \cdot T(r)\right).$$
Applying the inductive hypothesis,
$$T(n+1) \le 1 + \sum_{r=1}^ n \left(\mathbb{P}[R=r] \cdot \log_{\frac{4}{3}} r\right).$$
Applying Jensen’s Inequality,
$$T(n+1) \le 1 + \log_{\frac{4}{3}} \mathbb{E}[R].$$
Here $\mathbb{E}[R]$ is the expected number of   remaining hypotheses in $H$ that are consistent with the teacher's counter-example. 
By Lemma~\ref{sec4lem2},  the counter-example that is chosen  eliminates at least $\frac{1}{4}$ of the  hypotheses in expectation. Thus, $\mathbb{E}[R] \le \frac{3}{4} (n+1)$.  Thus,
$$T(n+1) \le 1 + \log_{\frac{4}{3}} \frac{3}{4} (n+1) = \log_{\frac{4}{3}} (n+1),$$
thus concluding the inductive step. Thus  it follows that Majority learning algorithm can learn the target concept in $\log_{\frac{4}{3}} |H|$ or $\mathcal{O}(\log{}|H|)$ expected rounds.
\end{proof}
Theorem~\ref{sec4t1} establishes that the expected learning time of the Majority algorithm is asymptotically optimal. However its expected learning time is $\log_{\frac{4}{3}}2 = 2.41$  times more than the $\log_{2} |H|$ expected learning time of the Max-Min algorithm~\citep{angluin2017power}. In addition there is a  gap between the learning time of the Majority algorithm and the $\log_{2} |H| -1 $  lower bound on the expected learning time for any LRC algorithm~\citep{angluin2017power}.  Can the analysis in Theorem~\ref{sec4t1}  for the Majority algorithm  be improved remains an interesting open question. One key difficulty for answering this question, as we show with Lemma~\ref{tightBound}, is that the bound in Lemma~\ref{sec4lem2} is tight. 
\begin{theorem}
\label{sec4t2}
Given a $\delta$ such that $0 > \delta > 1$, the Majority learning algorithm will terminate in $\mathcal{O}(\log{}{\frac{|H|}{\delta}})$ rounds with probability at least $1 - \delta$.
\end{theorem}

\begin{proof}
We use the same line of reasoning as given in~\citep{angluin2017power}.  Let $R_{i}$  be the number of consistent hypotheses remaining after $i$ rounds of the Majority algorithm. We first show by induction that $\mathbb{E}[R_i]  \le \left(\frac{3}{4}\right)^i \cdot |H|.$ By Lemma~\ref{sec4lem2},  the counter-example in a round of the Majority algorithm eliminates at least $\frac{1}{4}$ of the  hypotheses in expectation. Thus $\mathbb{E}[R_1] \le \frac{3}{4} \cdot |H|$ and the base case holds. Applying the inductive step for $i$ we get:
$$\mathbb{E}[R_{i+1}] = \mathbb{E}[\mathbb{E}[R_{i+1} \mid R_{i}]] \le \frac{3}{4} \cdot \mathbb{E}[R_{i}] \le \left(\frac{3}{4}\right)^{i+1} \cdot |H|, $$ thus completing the induction. The inequality $\mathbb{E}[R_{i+1} \mid R_{i}] \le \frac{3}{4} \cdot R_{i} $ used above follows from Lemma~\ref{sec4lem2}. 

After the $i$-th round the identity of the target concept is still unknown iff there are  two or more remaining consistent hypotheses or $R_i \ge 2$.  As $R_i$ is a non-negative random variable, we can apply  Markov inequality  to bound the probability of this event.
$$Pr(R_i \ge 2) \le  \frac{E[R_i]}{2} \le \frac{1}{2}\cdot \left(\frac{3}{4}\right)^{i} \cdot |H|.  $$  
Thus, for $i \ge \log_{\frac{4}{3}}{\frac{|H|}{\delta}}$ we have  $Pr(R_i \ge 2) \le \delta$ or $Pr(R_i \le 1) > 1- \delta$. Hence with probability   $ > 1 - \delta$ by round $ \log_{\frac{4}{3}}{\frac{|H|}{\delta}}$ there is at most  one    consistent hypotheses  remaining and the  Majority algorithm must  terminate having identified the target concept. 
\end{proof}
\begin{lemma}
\label{tightBound}
The bound  in Lemma~\ref{sec4lem2} on the fraction of hypotheses eliminated by the teachers counter-example in one round of the Majority algorithm is tight.
\end{lemma}
\begin{proof}
We  present an example to show that the bound is tight. Consider a binary matrix $H'$ over $2n+1$ columns in which the last column has only zeros while the values in the first $2n$ columns represent all possible  combinations of an equal number of zeros  and ones. In other word, $H'$ has ${2n\choose n}$ distinct rows each consisting of a unique combination of $n$ zeros  and $n$ ones in the first $2n$ columns and a zero in the last column. Note that $H$ also has an equal number of zeros and ones in each of the first $2n$ columns. Consider the row $r$ of $H'$ that has all zeros in the first $n$ columns and has all ones in the next $n$ columns. We construct  a new matrix $H$ from $H'$ by modifying the values in row $r$ of $H'$ as follows:  toggle  the value in the $n+1$-th column from one to zero and  toggle the value in the last column from zero to one. Note that this new matrix $H$ also has $n$ zeros and $n$ ones in every row. 

Let the learners concept class be $H$. Let $\mathbb{P}$  be the uniform probability distribution over $X$, the columns of $H$. Let $h'$ denote the hypothesis in $H$ that corresponds to row $r$ of $H$.   Consider the first round of the Majority algorithm. By Definition~\ref{sec4majdef},  the majority hypothesis \(\operatorname{MAJ}_H\) has all zeros. This is because in each of the columns of $H$, either the number of zeros and ones is the same and the tie breaking rule favors zeros over ones. Or, there are more zeros  (i.e. in $n+1$-th column and last column), in which case the column majority is zero. Therefore  all hypotheses $h \in H$ have the same probability \(\mathbb{P}(h(x) = \operatorname{MAJ}_H(x))\) as they all have $n$ zeros. By Definition~\ref{sec4hhatdef}, the majority hypothesis  \(\hat h \in H \) is therefore $h'$. This is because, ties for the majority hypothesis are broken in favor of the smallest one in a lexicographic sort. As  $h'$ has all zeros in its first $n+1$ columns while every other hypothesis $h \in H$  has at least one one in the  first $n+1$ columns, $h'$ is smaller than every other hypotheses in a lexicographic sort. Let the teachers' hypothesis $h^*$  differ from \(\hat h = h' \) in two columns: the first column (where it has a one) and the last column (where it has a zero). Note that such $h^*$ is in $H$ since $H$ was constructed from all possible ${2n\choose n}$ combinations of equal number of ones and zeros. Note also that the teachers counter-example is either for the value in the first column or for the value in the last column, each with probability $1/2$. If it is the former,  it eliminates half the hypothesis in $H$, since there is an equal number of ones and zeros in the first column. If it is the latter, it only eliminates one hypothesis (the learners hypothesis \(\hat h \)). Thus the expected fraction of hypotheses that get eliminated in the first round by the teachers counter-example is $$ \frac{1}{|H|}  \left( \frac{1}{2} \cdot \frac{|H|}{2} + \frac{1}{2} \cdot 1 \right) = \frac{1}{4} + \frac{1}{|H|},  $$ which approaches $1/4$ as $|H|$ becomes large.
\end{proof}
\begin{theorem}
\label{fast}
Each round of the majority algorithm can be implemented in time $O(|H||X|)$
\end{theorem}
\begin{proof}
By Definition~\ref{sec4majdef}, computation of $\operatorname{MAJ}_H(x)$ requires finding maximum values in each column of matrix $H$, which can be done in time $O(|H|)$ per column. Since there are $O(|X|)$ columns in $H$,  $\operatorname{MAJ}_H(x)$ can be computed in time  $O(|H||X|)$. Computation of
\(\mathbb{P}(h(x)) \) for any $h \in H$ takes time $O(|X|)$. Therefore, by Definition~\ref{sec4hhatdef},  the learners hypothesis \(\hat h\) can  be computed in time $O(|H||X|)$.   
\end{proof}
One of the primary advantages of the Majority algorithm (and also the randomized algorithm described in the next section) over the Max-Min algorithm is the per round running time for computing the next hypothesis. In the case of Max-Min algorithm, the computation of the weights of all the edges in the elimination graph~\citep{angluin2017power} entails a computation time of $O(|H|^2|X|)$. Compared to this the $O(|H||X|)$ per round computation time requirement of the Majority algorithm is  a significant improvement, particularly when $|H|$ is large.

\subsection{Randomized Learning Algorithm}
\label{ssrla}
In this section, an open problem posed by Angluin and Dohrn~\citep{angluin2017power} regarding the existence of an efficient randomized algorithm is solved. In this version of the LRC model, the teacher's target concept  $h^*$ is drawn from the learner's concept class $H$ according to a known probability distribution $\mathbb{Q}$. A Randomized learning algorithm is presented and mathematically analyzed and is shown to be asymptotically optimal. 
\begin{definition}
\label{qdef}
$\mathbb{Q}$ denotes a known teacher's probability distribution of drawing the target concept $h^*$ from $H$.
\end{definition}
This learning algorithm works as follows. In the first round of the Randomized learning algorithm (Algorithm~\ref{alg2}), the learner draws a hypothesis randomly from $H$  according to the distribution $\mathbb{Q}$. When presented with a counter-example, the learner updates $H$ by removing the hypotheses that disagree with the counter-example. The learner also updates the teacher's distribution $\mathbb{Q}$ to match the new $H$. The learner draws the hypothesis for the next round from the updated $H$ according to the updated distribution $\mathbb{Q}$. This process continues until the learner correctly learns the target concept.

\begin{definition}
\label{s}
In the Randomized learning algorithm (Algorithm~\ref{alg2}), the set of consistent hypotheses evolves over time by the sequence denoted by $H_1, H_2, H_3, \ldots $. Here $H_1 = H$ and $H_1 \supset H_2 \supset H_3 \ldots$. The corresponding evolution of the teacher's probability distribution $\mathbb{Q}$ by the learner over this set is denoted by the sequence $\mathbb{Q}_1, \mathbb{Q}_2, \mathbb{Q}_3, \ldots $, where $\mathbb{Q}_i$ is a distribution on the hypothesis set $H_i$. Furthermore, for the Randomized algorithm,  the sequence  of learner's hypothesis is denoted by $h_1, h_2, h_3, \ldots $ and the corresponding sequence of counter-examples is denoted by $x_1, x_2, x_3, \ldots $.
\end{definition}
\begin{definition}
For $i \ge 1$, $(h_i,x_i)$ denotes the pair of learner's hypothesis $h_i$  and the  corresponding  counter-example $x_i$ in round $i$ of the Randomized algorithm. For $i \ge 1$, $R_i = \{(h_1,x_1),(h_2,x_2) \ldots (h_i,x_i)\}$ denotes the  sequence  of these pairs  in the first $i$ rounds of the Randomized algorithm. $R_0 = \{ \}$ denotes the empty sequence of pairs. For $i \ge 1$, the notation $R_i = R_{i-1} + \{(h_i,x_i)\} $ is used to indicate that the sequence $R_i$ is a result of adding the pair $(h_i,x_i)$ to the sequence $R_{i-1}$.  
\end{definition}
The probability distributions $\mathbb{Q}_i$ are defined as follows.
$\mathbb{Q}_1 = \mathbb{Q}$ is set to the known teacher's distribution over $H$.
For $i \ge 1$,  $\mathbb{Q}_{i+1}$ is set to  the  teacher's posterior distribution over  $H_{i+1}$ given $R_i$. Specifically for $i \ge 1$ denote $q^i_j$ to be the probability the learner draws $h_j \sim \mathbb{Q}_{i}$. Then for $h_j \in H_{i+1}$,

$$q^{i+1}_j =  Pr(h^* = h_j \mid  R_i).$$
We now show how these probabilities can be recursively computed. 
\begin{lemma}
\label{lem3}
For $i \ge 1$, the teacher's posterior distribution over  $H_{i+1}$ given $R_i$ and hence the probability distribution $\mathbb{Q}_{i+1}$ can be recursively computed as 
$$
q^{i+1}_j =  \frac{ q^{i}_j  \mathbb{P}[x_i \mid h_j(x) \neq h_i(x) ] }{\sum_{h_k \in H_{i+1}} q^{i}_k  \mathbb{P}[x_i \mid h_k(x) \neq h_i(x) ]}
$$
\end{lemma}

\begin{proof}

Applying Bayes' theorem:
\begin{equation}
\label{bayes1}
q^{i+1}_j = Pr(h^* = h_j \mid R_i) = \frac{Pr(R_i \mid h^* = h_j)\cdot Pr(h^* = h_j)}{Pr(R_i)}.
\end{equation}

Consider $Pr(R_i \mid h^* = h_j)$ for $i \ge 1$. This can be written as 
$$Pr(R_{i-1} + \{(h_i,x_i)\} \mid h^* = h_j) = Pr(R_{i-1} \mid h^* = h_j) \cdot Pr(\{(h_i,x_i)\} \mid R_{i-1}  \wedge h^* = h_j).$$
In other words it is the conditional probability that 
$R_{i-1}$ is the sequence of pairs in the first $i-1$ rounds of the Randomized algorithm 
given that teacher's hypothesis is $h_j$ times the conditional probability that $\{(h_i,x_i)\}$ is the pair of hypothesis and  counter-example in the $i$-th round   given teacher's hypothesis $h_j$ and the sequence of pairs $R_{i-1}$ in the first $i-1$ rounds of the Randomized algorithm.
Applying Bayes' theorem  we get 
$$Pr(R_{i-1} \mid h^* = h_j)=  \frac{Pr(h^* = h_j \mid R_{i-1})Pr(R_{i-1})}{Pr(h^* = h_j)} = \frac{q^{i}_j Pr(R_{i-1})}{Pr(h^* = h_j)}.$$
The last equality follows from the definition of $q^{i}_j$.
Thus we have shown
\begin{equation}
\label{bayes2}
Pr(R_i \mid h^* = h_j) = \frac{q^{i}_j Pr(R_{i-1}) Pr(\{(h_i,x_i)\} \mid R_{i-1}  \wedge h^* = h_j)}{Pr(h^* = h_j)} 
\end{equation}
Substituting in Equation (\ref{bayes1}) we get:
\begin{equation}
\label{bayes3}
q^{i+1}_j =   \frac{ q^{i}_j Pr(\{(h_i,x_i)\} \mid R_{i-1}  \wedge h^* = h_j) Pr(R_{i-1})}{Pr(R_i)}.
\end{equation}
Note that $Pr(R_i) = \sum_{h_k \in H_{i+1}} Pr(R_i \mid h^* = h_k)Pr(h^* = h_k)$. Substituting in Equation (\ref{bayes3}) and applying Equation  (\ref{bayes2}) we get:

$$ q^{i+1}_j =  \frac{ q^{i}_j  Pr(\{(h_i,x_i)\} \mid R_{i-1}  \wedge h^* = h_j) Pr(R_{i-1})}{\sum_{h_k \in H_{i+1}} q^{i}_k  Pr(\{(h_i,x_i)\} \mid R_{i-1}  \wedge h^* = h_k) Pr(R_{i-1})}.$$

By simplifying we get
\begin{equation}
\label{bayes4}
q^{i+1}_j =  \frac{ q^{i}_j  Pr(\{(h_i,x_i)\} \mid R_{i-1}  \wedge h^* = h_j) }{\sum_{h_k \in H_{i+1}} q^{i}_k  Pr(\{(h_i,x_i)\} \mid R_{i-1}  \wedge h^* = h_k)}.
\end{equation}

Note that for any $h_k \in H_{i+1}$,
$$Pr(\{(h_i,x_i)\} \mid R_{i-1}  \wedge h^* = h_k) = Pr(h_i \mid R_{i-1}  \wedge h^* = h_k) \cdot Pr(x_i  \mid R_{i-1}  \wedge h_i \wedge h^* = h_k))$$
Here $Pr(h_i \mid R_{i-1}  \wedge h^* = h_k)$ is the conditional probability of the Randomized algorithm selecting hypothesis $h_i$ given the teacher's hypothesis $h_k$ and given its sequence of hypothesis and counter-example pairs $R_{i-1}$ in the first $i-1$ rounds. Also $Pr(x_i  \mid R_{i-1}  \wedge h_i \wedge h^* = h_k))$ is the conditional probability of  getting counter-example $x_i$ with the Randomized algorithm having selected hypothesis $h_i$ in its $i$-th round and given the teacher's hypothesis $h_k$ and given the sequence of hypothesis and counter-example pairs $R_{i-1}$ in its first $i-1$ rounds. Note that the probability $Pr(h_i \mid R_{i-1}  \wedge h^* = h_k)$  is independent of the choice of the teacher's hypothesis $h^* = h_k$ and it equals $Pr(h_i \mid R_{i-1})$. Thus,
substituting in Equation (\ref{bayes4}) and simplifying we have:
\begin{equation}
\label{bayes5}
q^{i+1}_j =  \frac{ q^{i}_j  Pr(x_i  \mid R_{i-1}  \wedge h_i \wedge h^* = h_j)) }{\sum_{h_k \in H_{i+1}} q^{i}_k  Pr(x_i  \mid R_{i-1}  \wedge h_i \wedge h^* = h_k))}.
\end{equation}
Note that the probabilities $Pr(x_i  \mid R_{i-1}  \wedge h_i \wedge h^* = h_k))$ are independent of $R_{i-1}$ and are distributed as $\mathbb{P}(h_i,h_k)$ (as defined in Definition \ref{pdef}). That is 
$$ 
Pr(x_i \mid R_{i-1}  \wedge h_i \wedge h^* = h_k)) = \mathbb{P}[x_i \mid h_k(x) \neq h_i(x) ]
$$

Substituting in Equation (\ref{bayes5}) the result follows. 
Thus, we have shown that the the  teacher's posterior distribution over  $H_{i+1}$ given $R_i$,  which is also the probability distribution  $\mathbb{Q}_{i+1}$ used by the Randomized algorithm, can be recursively computed from the probability distribution $\mathbb{Q}_{i}$ and the probability distribution  $\mathbb{P}(h_i,h_k)$ over the hypotheses $h_k \in H_{i+1}$.

\end{proof}

\begin{algorithm}[h!]
\label{alg2}
%
$r = 1$, $H_1 = H$, $\mathbb{Q}_1 = \mathbb{Q}$\\
 \While{true}{	
 	Draw the learner's hypothesis $h_r \in H_r$ randomly from $\mathbb{Q}_r$.\\
	Let $x_r$ be the counter-example returned by the teacher.\\
	\If{ there is no such counter-example}
	{Output $h_r$.}
	\Else{
	$H_{r+1} = H_r - \{h \in H_r \mid h(x_r) \neq h^*(x_r)\} $.\\
  Calculate $\mathbb{Q}_{r+1}$ as described in Lemma~\ref{lem3}.\\
	$r= r+ 1.$\\}
 }

 \caption{Randomized Learning Algorithm}
\end{algorithm}

The analysis for Algorithm~\ref{alg2} works as follows.
In Lemma~\ref{lem4}, a fact is proven using which it is shown in Lemma~\ref{lem5} that the expected fraction of eliminated hypotheses by the counter-example given is at least $\frac{1}{2}.$ 
In other words, $\mathbb{E}\left[|H_{i+1}| \right] \le |H_i|/2$. 
Using this result, it is shown in Theorem~\ref{t2} that the expected learning time of Algorithm~\ref{alg2} is $\mathcal{O}(\log{}|H|)$. 
\\The analysis below is for the $i$-th round of the algorithm and omits the index $i$ wherever possible.
\begin{definition}
For $h \in H_i$, define $V(h,x)$ as the fraction of hypotheses in $H_i$ that disagree with $h$ on example $x$. More formally, $$V(h,x) = \frac{|\{h' \in H_i \mid h'(x) \neq h(x)\}|}{|H_i|}.$$ 
\end{definition}
\begin{lemma}
\label{lemNew}
In any $i$-th round of the algorithm ($i\ge1$), for any two hypotheses $h_1,\ h_2 \in H_i$ where $h_1(x) \neq h_2(x)$, $V(h_1, x) + V(h_2, x) \ge 1$.
\end{lemma}
\begin{proof}
Note that $V(h_1, x) =$
$$ \frac{|\{h' \in H_i \mid h'(x) \neq h_1(x)\}|}{|H_i|} = $$
$$ \frac{|\{h' \in H_i \mid h'(x) \neq h_1(x) \land h'(x) = h_2(x)\}|}{|H_i|} + \frac{|\{h' \in H_i \mid h'(x) \neq h_1(x) \land h'(x) \ne h_2(x)\}|}{|H_i|}.$$
Likewise $V(h_2, x)=$
$$\frac{|\{h' \in H_i \mid h'(x) \neq h_2(x) \land h'(x) = h_1(x)\}|}{|H_i|} + \frac{|\{h' \in H_i \mid h'(x) \neq h_2(x) \land h'(x) \ne h_1(x)\}|}{|H_i|}.$$
Therefore $V(h_1, x) + V(h_2, x) =$ 
$$= \frac{|\{h' \in H_i \mid h'(x) \neq h_1(x) \land h'(x) = h_2(x)\}|}{|H_i|} + \frac{|\{h' \in H_i \mid h'(x) \neq h_2(x) \land h'(x) = h_1(x)\}|}{|H_i|} $$
$$+ \frac{|\{h' \in H_i \mid h'(x) \neq h_1(x) \land h'(x) \ne h_2(x)\}|}{|H_i|} + \frac{|\{h' \in H_i \mid h'(x) \neq h_1(x) \land h'(x) \ne h_2(x)\}|}{|H_i|}.$$
Note that the numerator of the first three fractions adds up to exactly $|H_i|$. Hence we have 
$$V(h_1, x) + V(h_2, x) = \frac{|H_i|}{|H_i|} + \frac{|\{h' \in H_i \mid h'(x) \neq h_1(x) \land h'(x) \ne h_2(x)\}|}{|H_i|} \ge 1.$$
\end{proof}
\begin{definition}
Define $E(h_1,h_2)$ to be the expected fraction of hypotheses that are eliminated from $H_i$ when the learner's hypothesis is $h_1 \in H_i$ and the target concept is $h_2 \in H_i$. 
$$E(h_1,h_2) = \sum_{x \in D(h_1,h_2)}{V(h_2, x) \cdot \mathbb{P}[x \mid h_1(x) \neq h_2(x) ]}.$$
\end{definition}
\begin{lemma}
\label{lem4}
In any $i$-th round of the algorithm ($i\ge1$), for any two hypotheses $h_1,\ h_2 \in H_i$ where $h_1 \neq h_2$, $E(h_1,h_2) + E(h_2,h_1) \ge 1$.
\end{lemma}
\begin{proof}
Recall from Definition~\ref{ddef} that $D(h_1, h_2) = \{x\in X \mid h_1(x) \neq h_2(x) \}$. Thus $D(h_1, h_2) = D(h_2, h_1).$ Thus
 \begin{align}
&E(h_1,h_2) + E(h_2,h_1) \nonumber \\
&\mbox{  } = \sum_{x \in D(h_1,h_2)}V(h_2, x)\cdot \mathbb{P}[x \mid h_1(x) \neq h_2(x) ] + \sum_{x \in D(h_2,h_1)}V(h_1, x)\cdot \mathbb{P}[x \mid h_2(x) \neq h_1(x) ] \nonumber \\
&\mbox{  } = \sum_{x \in D(h_1,h_2)}(V(h_2, x)\cdot \mathbb{P}[x \mid h_1(x) \neq h_2(x) ] + V(h_1, x)\cdot \mathbb{P}[x \mid h_1(x) \neq h_2(x) ]) \label{eqnLem4}
  \end{align}
Also since $h_1(x) \ne h_2(x)$, it follows by from Lemma~\ref{lemNew} that  for any $x \in D(h_1,h_2)$. $V(h_1, x) \ge 1 - V(h_2, x)$. Substituting in Equation (\ref{eqnLem4}) we get
 \begin{align*}
E(h_1,h_2) + E(h_2,h_1) &\ge \sum_{x \in D(h_1,h_2)}(V(h_2, x) + 1 - V(h_2, x))\cdot \mathbb{P}[x \mid h_1(x) \neq h_2(x) ] \\
&= \sum_{x \in D(h_1,h_2)}\mathbb{P}[x \mid h_1(x) \neq h_2(x) ] = 1.
  \end{align*}
The last equality follows from $D(h_1, h_2) = \{x\in X \mid h_1(x) \neq h_2(x) \}$.
\end{proof}
\begin{lemma}
\label{lem5}
In any $i$-th round of the algorithm ($i\ge1$) 
the expected fraction of hypotheses eliminated by the counter-example given is at least $\frac{1}{2}.$ In other words, $\mathbb{E}\left[|H_{i+1}| \right] \le |H_i|/2$.
\end{lemma}
\begin{proof}
Note that  in round $i$ of the Randomized algorithm the learner draws a hypothesis $h \sim \mathbb{Q}_i$, for $i \ge 1$. Note  that $\mathbb{Q}_i$ is also  the teacher's posterior distribution  over $H_i$ in round $i$ of the Randomized algorithm (Lemma~\ref{lem3}). 
Let $n = |H_i|$, and let $q_j$ denote the probability that $h_j \sim \mathbb{Q}_{i}$ (here we dropped the superscript $i$ in $q^i_j$ for ease of exposition). Thus, for hypotheses $h_j, h_k \in H_i$, $q_j$ is  the probability of learner drawing hypothesis $h_j$ and $q_k$ is  the probability of  $h_k$ being teacher's hypothesis in round $i$ of the Randomized algorithm.
Therefore, the expected fraction of hypotheses eliminated by the counter-example in round $i$ of the Randomized algorithm is 
$$\sum_{j=1}^{n}\sum_{k=1}^{n}q_jq_kE(h_j,h_k) . $$
By Lemma~\ref{lem4}, and using the fact that $E(h,h) = 1$ for any $h \in H_i$ the expected fraction of hypotheses eliminated by the counter-example in round $i$ of the Randomized algorithm,
 \begin{align*}
  & \ge \sum_{j=1}^{n}{q_j}^2 + \sum_{j=1}^{n}\sum_{k>j}^{n}q_jq_k 
  =(q_1 + q_2 + q_3 + ... + q_n)^2 - \sum_{j=1}^{n}\sum_{k>j}^{n}q_jq_k \\
 & \ge (q_1 + q_2 + q_3 + ... + q_n)^2 - \frac{(q_1 + q_2 + q_3 + ... + q_n)^2}{2} 
  = (1-\frac{1}{2}) = \frac{1}{2}. 
 \end{align*}
\end{proof}
\begin{theorem}
\label{t2}
In the setting where the target concept is drawn from $\mathbb{Q}$ and the counter-examples are drawn from $\mathbb{P}$, Algorithm~\ref{alg2} only needs to see an expected
$\log{}|H|$
counter-examples to learn $h^*$.
\end{theorem}
\begin{proof}
From Lemma~\ref{lem5} it follows that in  any $i$-th round of the Randomized algorithm ($i\ge1$) at least $\frac{1}{2}$  of the remaining hypotheses get eliminated in expectation. 
Thus, by Theorem 21 of Angluin and Dohrn~\citep{angluin2017power} it follows that the Randomized algorithm (Algorithm~\ref{alg2}) can learn the teacher's hypothesis in 
$\log{}|H|$
expected rounds.
\end{proof}

\subsection{Upper Bound on Learning Time of Arbitrary Learning Algorithm}
\label{ssala}
In this section, an Arbitrary learning algorithm (Algorithm~\ref{alg3}) is analyzed in the PAC-LRC model where the learner is allowed to pick any consistent hypothesis in every round. Just as in exact learning, in the  PAC-LRC model, in  each round  a randomly drawn counter-example is returned to the learner. However the two differ in their terminating conditions which  determine when  learning  ends for the learner. 
Recall that in the PAC-LRC model with parameters $\epsilon$ and $\delta$, learning ends if the probability of elimination of all $\epsilon$-\textit{bad} hypotheses (hypothesis that differs with the target concept in a region that has total probability at most $\epsilon$) in $H$ exceeds $1-\delta$.
The only exception is if the learner  presents a  hypothesis that is not $\epsilon$-\textit{bad}, in which case the teacher does not return a counter-example and the learning ends right away.   

In the following, it is shown that with probability at least $1 - \delta$, the Arbitrary learning algorithm terminates within $\mathcal{O}(\frac{1}{\epsilon}\log{\frac{|H|}{\delta}})$ rounds.

\begin{algorithm}[h!]
\label{alg3}
 $i=1.$ \\
 \While{true}{
   Pick $h_i$ to be any arbitrary hypothesis in $H$. \\
	Let $x_i$ be the counter-example returned by the teacher.\\
	\If{ there is no such counter-example}
	{Output $h_i$. \\}
	\Else
	{Eliminate the set of hypotheses $\{h \in H \mid h(x_i) \neq h^*(x_i)\}$ from $H$.\\}	
	$i=i+1.$\\
 }
 \caption{Arbitrary Learning Algorithm}
\end{algorithm}
\begin{definition}
Let the target hypothesis be $h^*$. Let  $h_i$ denote the learner's hypotheses at round $i$. That is the learner's hypotheses is the sequence $h_1, h_2 \ldots $. Let the sequence of counter-examples received by the learner be denoted by $x_1, x_2 \ldots$.
\end{definition}
\begin{definition}
\label{weightdef}
Let the \textit{weight} of a column $x \in X$ at the start
of any round $i\ge 1$ be denoted as $W_i(x)$, and let $W_i(S) = \sum_{x \in S}{W_i(x)}$. For hypothesis $h$,  $W_i(h)= W_i(D(h, h^*))$ denotes the total weight on all the columns on which $h$ differs from $h^*$. For all $x \in X$, define $W_1(x) = 0$, and allow the weight of a column to be incremented in each round by the probability that the column is chosen to be a counter-example. More formally for $n \ge 2$,
\[
W_{n}(x) = W_{n-1}(x) + \mathbb{P}[x \mid h_{n-1}(x) \neq h^*(x)].
\]
Or
\[
W_{n}(x) = \sum_{i=1}^{n-1} \mathbb{P}[x \mid h_i(x) \neq h^*(x)].
\]
Note that in each round $i \ge 1$, $W_{i+1}(X) = W_{i}(X) + 1$. This is because 
$$\sum_{x \in X} \mathbb{P}[x \mid h_i(x) \neq h^*(x) ] = 1$$
\end{definition}
\begin{definition}
\label{errdef}
Let for $i \ge 1$, $E_i(h)$  be the probability that a hypothesis $h$ is eliminated in round $i$. Thus, $E_i(h)$ is the probability that on the counter-example $x_i$,  hypothesis $h$ differs from $h^*$. 
More formally, 
$$E_i(h) = {\mathbb{P}[h(x) \neq h^*(x) \mid h_i(x) \neq h^*(x) ] }.$$
\end{definition}
Note that for any $h\in H$ and any round $n \ge 1$, by Definition~\ref{weightdef} and Definition~\ref{errdef}, 
\[\sum_{i = 1}^{n}{E_i(h)} = \sum_{i = 1}^{n}{\phantom{100} {\mathbb{P}[h(x) \neq h^*(x) \mid h_i(x) \neq h^*(x) ] }}\]
\[\phantom{100}\phantom{100}\phantom{100}\phantom{100}= \sum_{x \in D(h,h^*)}{\phantom{100}\sum_{i = 1}^{n}{\mathbb{P}[x \mid h_i(x) \neq h^*(x) ]}} = W_{n+1}(D(h, h^*)) = W_{n+1}(h).\]
\begin{lemma}
\label{elimlem}
Let $\theta = \ln(\frac{|H|}{\delta}).$ At the beginning of any round $n \ge 1$, consider a hypothesis $h\in H$ for which $W_n(h) >
\theta$. The probability that $h$ is still consistent, in other words $h$ has not been eliminated already, is at most $\frac{\delta}{|H|}$. 
\end{lemma}
\begin{proof}
In the $i$-th round of the algorithm the counter-example $x_i$ is drawn with the probability distribution $\mathbb{P}[x \mid h_i(x) \neq h^*(x) ] $. The probability that on this counter-example   hypothesis $h$ is inconsistent with $h^*$ is therefore  ${\mathbb{P}[h(x) \neq h^*(x) \mid h_i(x) \neq h^*(x) ] },$ which by Definition~\ref{errdef} equals $E_i(h)$. Thus, with probability  $1-E_i(h)$, hypothesis $h$ is not eliminated in round $i$. 
Thus, the probability that hypothesis $h \in H$ with $W_n(h) > \theta$ has not been eliminated in the first $n-1$ rounds is:
$$\prod_{i = 1}^{n-1}{(1-E_i(h))} \le \prod_{i = 1}^{n-1}{(1-\frac{\sum_{i = 1}^{n-1}{E_i(h)})}{n-1}} = \prod_{i = 1}^{n-1}{(1-\frac{W_n(h)}{n-1})} \le e^{-W_n(h)} \le e^{-\ln(\frac{|H|}{\delta})} = \frac{\delta}{|H|}.$$
\end{proof}
\begin{definition}
\label{deflight}
Let a \textit{threshold} $\theta^*(x) = \ln(\frac{|H|}{\delta})\cdot\frac{2\mathbb{P}(x)}{\epsilon}$ be defined over every $x \in X$ based on the probability distribution $\mathbb{P}$. A column is considered \textit{light} at the start of any round $i$ if $W_i(x) \le \theta^*(x)$ and \textit{heavy} otherwise. Let $L_i \subset X$ denote the set of light columns and $B_i \subset X$ denote the set of heavy (or bulky) columns at the start of round $i$. 
\end{definition}
\begin{lemma}
\label{halflem}
Let the learner's hypothesis $h_i$ in round $i \ge 1$ be $\epsilon$-\textit{bad}. That is $\mathbb{P} [h_i(x) \neq h^*(x) ] \ge \epsilon $. Let $h_i$ also satisfy that at  the start of round $i$ its total weight $ W_{i}(h_i) \le \ln(\frac{|H|}{\delta})$. Then,
the weight of light columns $L_i$ should increase by at least half in  round $i$. In other words
$$W_{i+1}(L_i) \ge W_i(L_i) + \frac12.$$
\end{lemma}
\begin{proof}
Assume for the sake of contradiction that $$W_{i+1}(L_i) < W_i(L_i) + \frac12.$$ $W_{i+1}(X) = W_i(X) + 1$ and $B_i + L_i = X$ together imply that
\begin{equation}
\label{eqn:1}
    W_{i+1}(B_i) > W_i(B_i) + \frac12.
\end{equation}
Note that  only the weights of the columns in the set $D(h_i, h^*)$ change in this round as the counter-example $x_i$ is drawn from the probability distribution $\mathbb{P}[x \mid h_i(x) \neq h^*(x) ]$. Specifically only the weight of the heavy columns in the set $D(h_i, h^*). \cap B_i$ can change  in this round. Applying Definition~\ref{weightdef} we therefore have,
\begin{equation}
\label{eqn:2}
W_{i+1}(B_i) - W_i(B_i) =  \sum_{x\in D(h_i, h^*). \cap B_i }\mathbb{P}[x \mid h_i(x) \neq h^*(x) ].
\end{equation}
 From Equations (\ref{eqn:1}) and (\ref{eqn:2}) we get, 
 \begin{equation}
\label{eqn:3}
\sum_{x\in D(h_i, h^*) \cap B_i }{\mathbb{P}[x \mid h_i(x) \neq h^*(x) ]} =  \mathbb{P}[D(h_i, h^*) \cap B_i \mid h_i(x) \neq h^*(x) ] > \frac12.
\end{equation}
The total weight on $h_i$ at the beginning of round $i$ satisfies
$$
W_i(h_i) = W_i(D(h_i, h^*)) \ge  W_i( D(h_i, h^*) \cap B_i )
$$
Since the weights of all $x \in B_i$, at the beginning of round $i$, is at least $\theta^*(x) = \ln(\frac{|H|}{\delta})\cdot\frac{2\mathbb{P}(x)}{\epsilon}$ (Definition~\ref{deflight}), the total weight on $h_i$ at the beginning of round $i$ satisfies
\begin{equation}
\label{eqn:4}
W_i(h_i)  \ge  W_i( D(h_i, h^*) \cap B_i )> \ln(\frac{|H|}{\delta})\cdot\frac{2\cdot \mathbb{P}[D(h_i, h^*) \cap B_i]}{\epsilon}
\end{equation}
Since $\mathbb{P} [h_i(x) \neq h^*(x) ] \ge \epsilon $, and since $\mathbb{P}[D(h_i, h^*) \cap B_i \mid h_i(x) \neq h^*(x) ] > \frac12$ (Equation~\ref{eqn:3}) 
we have 
$$ \mathbb{P}[D(h_i, h^*) \cap B_i)] > \frac{\epsilon}{2}. $$
Combining with Equation (\ref{eqn:4}) we get, the total weight on $h_i$ at the beginning of round $i$ satisfies
$$W_i(h_i) > \ln(\frac{|H|}{\delta})\cdot\frac{2\cdot \mathbb{P}[D(h_i, h^*) \cap B_i]}{\epsilon}> \ln(\frac{|H|}{\delta})\cdot\frac{2 \cdot \epsilon}{\epsilon \cdot 2} = \ln(\frac{|H|}{\delta}).$$
This is a contradiction because we assumed $W_i(h_i) < \ln(\frac{|H|}{\delta})$.
\end{proof}
\begin{lemma}
\label{allheavy}
If for all $i \ge 1$ the learner's hypothesis $h_i$ satisfies  $ W_{i}(h_i) \le \ln(\frac{|H|}{\delta})$
then within  $\mathcal{O}(\frac{1}{\epsilon}\log{\frac{|H|}{\delta}})$ rounds  all $x \in X$ become heavy or the algorithm terminates because  the learner presents a hypothesis that is not $\epsilon$-\textit{bad}. 
\end{lemma}
\begin{proof}

Note that it is enough to bound   the  number of rounds it takes for all $x \in X$ to become heavy by assuming that in every round the  learner's hypothesis is $\epsilon$-\textit{bad}.  
Thus,  in each round $i$ we can assume that the learner's hypothesis is  $\epsilon$-\textit{bad} and its weight is at most $\ln(\frac{|H|}{\delta})$ and hence  Lemma~\ref{halflem} applies. 
Hence in each round $i \ge 1$, the total weight of light columns  increase by at least one half. We will assume this in the proof below.


Note that all columns start out being light initially (since $W_1(x) = 0$, for all $x \in X$). Consider the situation  where a  column $x$ that starts out light in round $i$  turns into a heavy column during round $i$. For $x$ this change can happen only one time since once  $x$ becomes heavy it cannot become light again. The increase in weight that tips $x$  from light to a heavy column in round $i$ is upper bounded by $\frac{\mathbb{P}(x)}{\epsilon}$. This is because  $\mathbb{P} [h_i(x) \neq h^*(x) ] \ge \epsilon $, and therefore $\mathbb{P}[x \mid h_i(x) \neq h^*(x) ] \le \frac{\mathbb{P}(x)}{\epsilon}.$ Thus, for each light column $x$ at most $\frac{\mathbb{P}(x)}{\epsilon}$ of its weight increase can contribute to its transformation from a light to a to heavy column. For all light columns this total contribution of their weight increase towards their transformation is therefore upper bounded by 
$$\sum_{x \in X} \frac{\mathbb{P}(x)}{\epsilon} = \frac{1}{\epsilon}.$$ 

Consider a round $i \ge 1$ at  the end of which   there are still some light columns left. From Lemma~\ref{halflem}  it follows that in each of the rounds $1, 2 \ldots i$  the total weight of the light columns would have increased by at least half. Thus, in these $i$ rounds at least $\frac{i}{2}$ of the   weight increase got  applied to columns in the rounds in which they were light. Out of this, as shown above, at most $\frac{1}{\epsilon}$ weight increase contributed to the transformation of light columns to heavy columns. Hence at least $\frac{i}{2} - \frac{1}{\epsilon}$ of this weight increase contributed to the increase in weight of the columns during rounds in which  they stayed light. 
By definition~\ref{deflight}, when a column $x$ is light its weight is at most  $ \theta^*(x) = \ln(\frac{|H|}{\delta})\cdot\frac{2\mathbb{P}(x)}{\epsilon}$. Thus, out of  $\frac{i}{2} - \frac{1}{\epsilon}$ only $\ln(\frac{|H|}{\delta})\cdot\frac{2\mathbb{P}(x)}{\epsilon}$ weight could have been applied towards column $x$. Hence, in total out of $\frac{i}{2} - \frac{1}{\epsilon}$ only
$$\sum_{x \in X} \ln(\frac{|H|}{\delta})\cdot\frac{2\mathbb{P}(x)}{\epsilon}  = \ln(\frac{|H|}{\delta})\cdot\frac{2}{\epsilon}$$
weight increase could have been applied towards all the columns $x \in X$. Thus, we must have 
$$\frac{i}{2} - \frac{1}{\epsilon} \le  \ln(\frac{|H|}{\delta})\cdot\frac{2}{\epsilon}.$$
Or, 
$$  i  \le \frac{4\cdot\log{\frac{|H|}{\delta}}+2}{\epsilon}.$$ 
Thus,  at the end of $N=i+1$ rounds there cannot be any more light columns left. Here
$$  N =  \frac{4\cdot\log{\frac{|H|}{\delta}}+2}{\epsilon} + 1 = \mathcal{O}(\frac{1}{\epsilon}\log{\frac{|H|}{\delta}}).$$ 
This establishes the Lemma.
\end{proof}
\begin{theorem}
With probability at least $1 - \delta$ the Arbitrary learning algorithm terminates in $\mathcal{O}(\frac{1}{\epsilon}\log{\frac{|H|}{\delta}})$ rounds.
\end{theorem}
\begin{proof}
We show that with probability at least $1 - \delta$ the Arbitrary algorithm cannot run for more than $N$ rounds where
$$  N =  \frac{4\cdot\log{\frac{|H|}{\delta}}+2}{\epsilon} + 1 $$.
Let us assume that is not the case and the algorithm runs for $m=N+1$ rounds with probability $1 - \delta$.
Note that all the learner's hypotheses $h_1, h_2 \ldots h_{m-1}$ have to be $\epsilon$-\textit{bad}, otherwise the algorithm would have ended earlier.  From Lemma~\ref{elimlem} it  follows that the probability of $W_{i}(h_i) > \ln(\frac{|H|}{\delta})$ is  at most $ \frac{\delta}{|H|}$, for any learner's hypothesis $h_i$. This is because each hypothesis $h_i$ is consistent in the beginning of the $i$-th round in which it got selected by the algorithm. Thus, by applying the union bound we get that with probability at least $ 1 - \frac{\delta}{|H|}  \cdot m \ge 1 - \delta$, each of the learner's hypothesis $h_i$ satisfies $ W_{i}(h_i) \le  \ln(\frac{|H|}{\delta})$. Here we used  the fact that $m \le |H|$, as any PAC-LRC learning algorithm must always terminate within $|H|$ rounds. We now continue the proof under the assumption that the learner's hypothesis $h_1,h_2 \cdots h_{m}$ satisfy $ W_{i}(h_i) \le  \ln(\frac{|H|}{\delta})$.   

Consider the first 
$$m-1 = N = \frac{4\cdot\log{\frac{|H|}{\delta}}+2}{\epsilon} + 1 $$
rounds. In these rounds  the learner's hypotheses $h_i$ are all $\epsilon$-\textit{bad} and their weights satisfy $W_{i}(h_i) \le \ln(\frac{|H|}{\delta})$.
Thus, Lemma~\ref{allheavy} applies and it follows that  all the columns become heavy at the end of $m-1=N$ rounds. 

Consider the $N$-th round. By Definition~\ref{weightdef}, $W_N(h_N) = W_N(D(h_N, h^*))$. Since every column in round $N$ is heavy, by Definition~\ref{deflight},  $$W_N(D(h_N, h^*)) > \ln(\frac{|H|}{\delta})\cdot\frac{2\cdot \mathbb{P}(D(h_N, h^*) )}{\epsilon}.$$  
Also since $h_N = h_{m-1}$ is $\epsilon$-\textit{bad}, we have $\mathbb{P}(D(h_N, h^*) ) \ge \epsilon$. Thus, 
$$W_N(h_N) > \ln(\frac{|H|}{\delta})\cdot\frac{2\cdot \mathbb{P}(D(h_N, h^*) )}{\epsilon}\ge \ln(\frac{|H|}{\delta})\cdot\frac{2 \cdot \epsilon}{\epsilon} > \ln(\frac{|H|}{\delta}).$$
This is a contradiction because  $W_N(h_N) = W_{m-1}(h_{m-1}) \le \ln(\frac{|H|}{\delta})$. Thus, under the assumption that all learner's hypothesis $h_1,h_2 \cdots h_{m}$ satisfy $ W_{i}(h_i) \le  \ln(\frac{|H|}{\delta})$, the Arbitrary algorithm must terminate in $N = \mathcal{O}(\frac{1}{\epsilon}\log{\frac{|H|}{\delta}})$ rounds. Since this assumption holds with  probability at least $1 - \delta$ the result follows.

\end{proof}
\section{Simulation Results}
In this section we present the results of testing the performance of our interactive LRC algorithms on a 
synthetic data-set.
We first generate a  data-set of  movie classifications by users (who either like or dislike each movie). We then use our algorithms to interactively learn which movies a new user would like to watch, based on this data-set. In particular, we seek to match the new user with a user in the data-set who classifies all movies in exactly the same way. We assume that such a user exists in the data-set. 

In this setting, the new user is the teacher, and the learner/algorithm's goal is to figure out which movies the teacher would like to watch. Here, the set of samples $X$ corresponds to the set of movies. Additionally, each user in the data-set is represented by a unique hypothesis in the learner's concept class $H$. These hypotheses $h \in H$ represent how the user classifies all the movies. For instance, $h(x)=1$ if and only if the corresponding user likes to watch movie $x \in X$. Note that the teacher's  target concept $h^*$ is also a hypothesis in $H$. $h^*$ is how the new user classifies movies. An interaction between the learner and the teacher involves the learner recommending movies to the teacher based on their current hypothesis and the teacher returning as feedback a randomly selected movie on which the learner's movie recommendation is wrong. The learner's goal is to learn the target concept without making too many mistakes.  

We generated different data-sets consisting of $n=5,000$ movies by varying the number of hypotheses (users) $|H|$  from $10,000$ to $100,000$. Within each data-set, we selected the hypotheses in $H$ to form a small number of clusters based on their similarity. This is to model that in real life, many users have similar tastes (for example, based on movie genre), and can be categorized into a small number of groups based on their preferences. In our data-set, we varied the number of clusters  from $100$ to $1,000$. To form these clusters of hypotheses, we first selected a representative hypothesis for each cluster. These representative hypotheses $h$ themselves were formed by randomly selecting the values  $h(x)$ for all $x \in X$. We selected other hypotheses in the cluster by randomly flipping up to $p$ values of the clusters representative hypothesis. Here, $p=20$.  

We tested the performance of the Majority learning algorithm using the generated data-set. In addition, we tested an Arbitrary learning algorithm that arbitrarily ordered the hypotheses and then always selected the $k$-th consistent hypothesis in $H$. We conducted three trials where the algorithm always picked the first, middle, and then last hypotheses. We ran each trial 10 times and we took an average of the results.  

Figure~\ref{figure1} shows the performance of  the Majority learning algorithm on this problem in comparison with both the theoretical bound as well as with the performance of the Arbitrary learning algorithm. We did not include the constant factor in the theoretical bound, since it was greater than one, and would have only made the bound worse. Also, for the theoretical bound, all logs were calculated to the base $2$. For the Arbitrary learning algorithm in this test, we allowed up to $1\%$ error ($\epsilon = 0.01$) with $90\%$ confidence ($\delta = 0.1$).  

In figure~\ref{figure1}, the X-axis shows the number of hypotheses in $H$ and the Y-axis shows the expected number of interactions with the teacher or the number of mistakes made by the learner in learning the target concept. As shown in figure~\ref{figure1}, the Majority learning algorithm never makes mistake on more than $10$ out of $5,000$ movies on the average. In addition, its performance is almost $2$ times better than its theoretical bound and it significantly outperforms the arbitrary algorithm.

\begin{figure}
    \centering
    \includegraphics[width=5.0in]{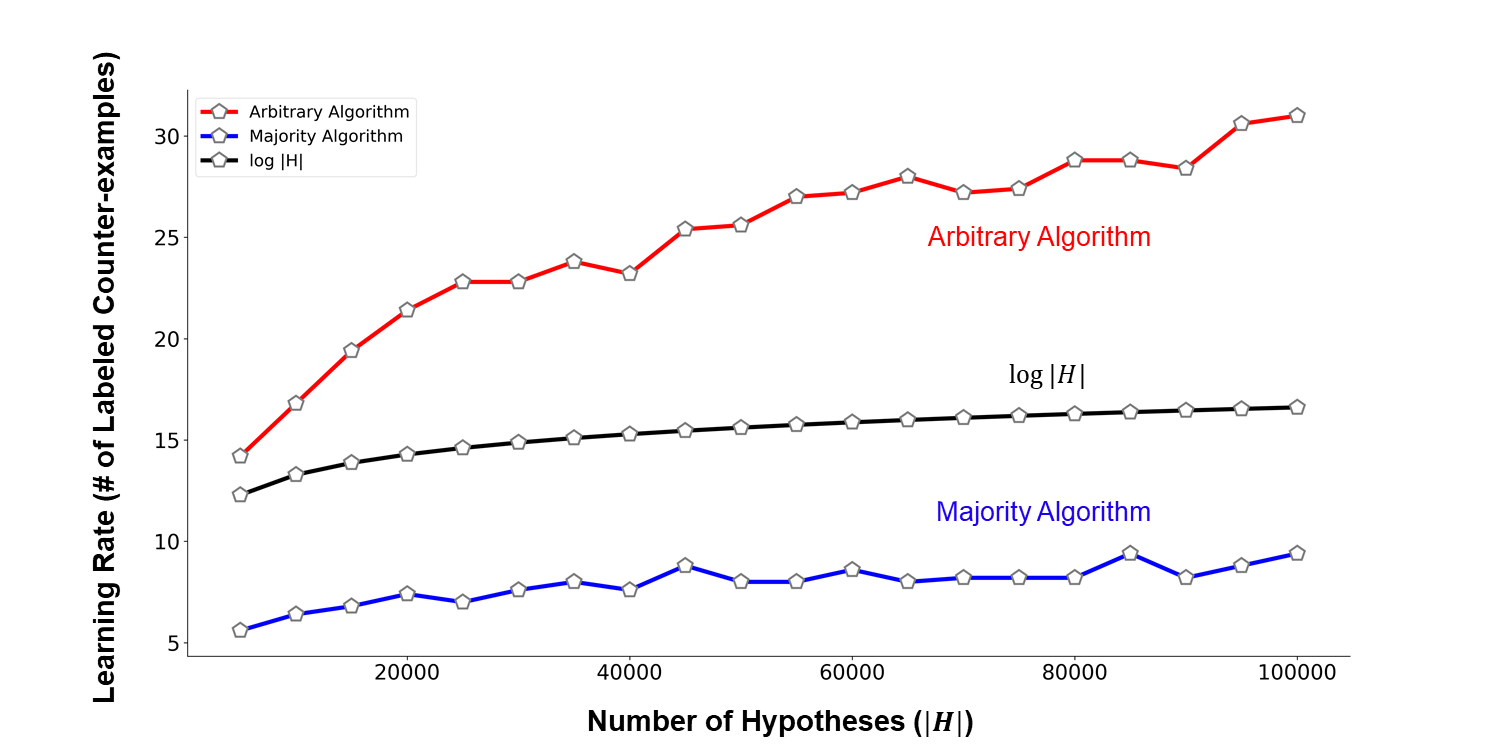}
    \caption{Performance of Majority Learning Algorithm}
    \label{figure1}
\end{figure}

In Figure~\ref{figure2}, we compared the performance of an arbitrary interactive learning algorithm to the theoretical bound of a non-interactive PAC learner. For this evaluation, as shown on the X-axis, we varied the error $\epsilon$ from $0.01$ to $0.6$.  The value of $\delta$ was held constant at $0.1$. We tried two different concept classes $H$ for  $|H| = 10,000$ and $|H| = 100,000$. For the calculation of the theoretical bound of a non-interactive PAC learner, all logs were calculated to the base $2$.

For the Arbitrary learning algorithm, the Y-axis shows the expected number of interactions with the teacher or the number of mistakes made by the learner in learning the target concept. For the non-interactive PAC learner, the Y-axis shows the number of examples required to guarantee the particular error tolerance and confidence level. 

As expected, at high error tolerance, both the interactive and non-interactive algorithms require less counter-examples (or examples). At low error tolerance, they require more counter-examples (or examples). However, it can be seen that the interactive Arbitrary algorithm significantly outperforms the non-interactive PAC learner particularly when the error tolerance ($\epsilon$) is small. This contrasts the fact that the theoretical bounds are the same for both kind of learners. 

 \begin{figure}
    \centering
    \includegraphics[width=5.0in]{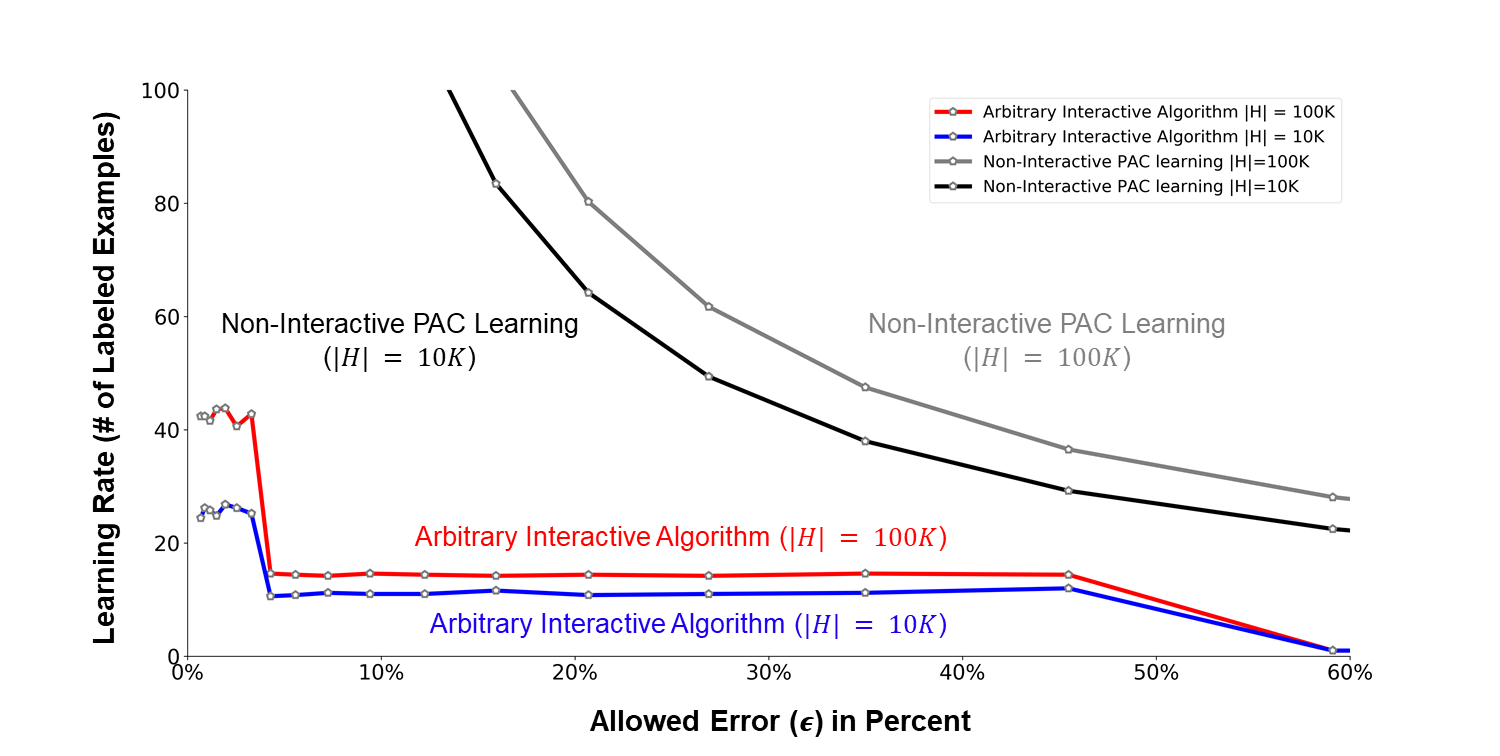}
    \caption{Performance of Arbitrary Learning Algorithm}
    \label{figure2}
\end{figure}
\label{simulation}

\section{Discussion}
\label{discussion}
These results demonstrate the potential benefits and limitations of using random counter-examples as a form of feedback in interactive learning. When learners are intelligent, and use the Majority learning algorithm (Algorithm~\ref{alg1}), they can learn by seeing an expected $\mathcal{O}(\log{}|H|)$ counter-examples. As was previously shown~\citep{angluin2017power}, 
when counter-examples are not chosen randomly, it is difficult for the learner to learn some concept classes. This happens, for instance, when $H$ is simply the $n \times n$ identity matrix. In this case, for any target concept $h^*$ and consistent hypothesis $h \ne h^*$, the teacher will have a choice between exactly two counter-examples. One counter-example will eliminate all hypothesis but $h^*$, and the other \textit{bad} counter-example will just eliminate $h$. An adversarial teacher could simply pick the \textit{bad} counter-example every single round and thus it would take $\Omega(|H|)$ time to learn $h^*$ without random counter-examples. While this problem of needing random counter-examples was previously solved by a well known halving algorithm based on majority vote~\citep{littlestone1988learning, bf-pgrf-72, angluin1988queries}, that algorithm required that the learner be allowed to make improper queries. On the other hand, as our work shows that teacher who gives random counter-examples solves this problem without having to make any such sacrifices.


The $\mathcal{O}(\frac{1}{\epsilon}\log{\frac{|H|}{\delta}})$ upper bound on the adversarial learner shows that the interactive LRC learner performs no worse than, and achieves the same bound as, the non-interactive PAC learner~\citep{kearns1994introduction}. This is somewhat surprising because intuitively, providing specific feedback in the form of counter-examples would seem more valuable to the learner than providing randomly sampled examples as is done in the PAC model~\citep{kearns1994introduction}. It seems that the reason that the bound was not improved was that the learner's ability to be adversarial, or impede the learning process was much more pronounced in the LRC setting than the PAC setting. The reason for this was that in the LRC setting, the learner could choose specific hypothesis on which the teacher's random counter-example would generally make little progress in eliminating $\epsilon$-\textit{bad} hypotheses.


\section{Conclusion and Future Work}
\label{conclusion}
In this work we provided simple and efficient algorithms for interactively learning non-binary concepts in the recently proposed setting of exact learning from random counter-examples (LRC). One such algorithm is based on majority vote and the other works by randomly selecting hypotheses from a probability distribution over target concepts that is evolved over time. Both these algorithms are shown to have the fastest possible $\mathcal{O}(\log{}|H|)$ expected learning time and  entail significantly lower computation time  than previously known algorithms. We also
provided an analysis that shows that interactive LRC learning, regardless of the learning algorithm, is at least as efficient as non-interactive Probably Approximately Correct (PAC) learning. \\
Our future goal is to improve the efficiency of these algorithms on other cost measures.
Throughout this paper, our focus has been  on minimizing the learning complexity of the  algorithms, which is  the number of counter-examples needed by the learner in order to correctly identify the target concept. However, 
calculating the majority hypothesis or updating the teacher's distribution $\mathbb{Q}$ to match the new $H$,
can require iterating over every remaining consistent hypothesis in $H$, in every round. This can take time $O(|H||X|)$, thus making it computationally  expensive, especially since the number of hypothesis in $H$ may grow exponentially in the number of examples (the set of columns $X$ of $H$). To address the high computational overhead of these tasks, we plan to explore alternative approaches for carrying out these tasks such as by using sampling to trade accuracy for efficiency. 



\acks{I would like to thank my mentor Professor Daniel Hsu who guided my research in the right direction by providing me with background material in the field of computational machine learning, validating the correctness of my proofs, and assisting me in using mathematical notation when writing this paper. Professor Hsu also helped me formulate theoretical models such as the PAC-LRC interactive learning model which served as the framework for my analysis of the upper bound of an arbitrary LRC algorithm.}


\bibliography{sample}

\end{document}